\newif\ifarxiv 
\arxivtrue

\newif\ifreview 
\reviewfalse

\newif\iffinal 
\finaltrue

\ifarxiv
\documentclass{article}
\else
\documentclass[envcountsame]{llncs}
\fi

\usepackage[utf8]{inputenc}
\usepackage[utf8]{inputenc}

\ifarxiv
\usepackage[a4paper,left=2cm,right=2cm,top=2cm,bottom=2cm]{geometry}
\usepackage[onehalfspacing]{setspace}
\usepackage{amsthm}
\usepackage{authblk}
\usepackage[hidelinks]{hyperref}
\fi

\usepackage[numbers]{natbib}
\usepackage{amsmath,amssymb,mathtools,dsfont}
\usepackage{esvect}

\usepackage{booktabs} 
\usepackage{xspace}
\usepackage{graphicx}
\usepackage{xspace}
\usepackage{url}
\usepackage{enumitem}
\setlist[enumerate,1]{label=\textup{(\roman*)}}

\usepackage{caption} 
\usepackage{subcaption}

\usepackage{booktabs}
\usepackage{svg}
\usepackage{hyperref}

\usepackage[algo2e,ruled,vlined,linesnumbered]{algorithm2e} 
\SetAlgoSkip{}
\DontPrintSemicolon

\ifarxiv 
\usepackage[appendix=inline]{apxproof}  
\RenewEnviron{proofsketch}{}
\RenewEnviron{toappendix}{}
\else
\ifreview
\usepackage[appendix=append]{apxproof}  
\else
\usepackage[appendix=strip]{apxproof}   
\fi
\usepackage{xparse} 
\RenewDocumentEnvironment{proofsketch}{o} 
{%
  \par\noindent\textit{Proof sketch%
  \IfValueT{#1}{ (#1)}.}\quad
}
{\par}
\fi

\allowdisplaybreaks[3] 

\usepackage{tikz}
\usetikzlibrary{shapes,arrows,angles}
\usetikzlibrary{decorations}
\usetikzlibrary{decorations.markings}
\usetikzlibrary{plotmarks}
\usetikzlibrary{calc}
\usetikzlibrary{mindmap}
\usetikzlibrary{shadows}
\usetikzlibrary{backgrounds}
\usetikzlibrary{patterns}
\usetikzlibrary{shapes.symbols}

\usepackage{pgfplots}
\usepackage{pgfplotstable}
\usepgfplotslibrary{statistics}
\pgfplotsset{compat=1.18}

\ifarxiv
\newtheorem{theorem}             {Theorem}
\newtheorem{lemma}      [theorem]{Lemma}

\newtheorem{definition} [theorem]{Definition}

\fi

\newtheoremrep{theorem}{Theorem}
\newtheoremrep{lemma}     [theorem]{Lemma}
\newtheoremrep{definition}[theorem]{Definition}
\newtheoremrep{conjecture}[theorem]{Conjecture}
\newtheoremrep{corollary} [theorem]{Corollary}
\newtheoremrep{proposition}[theorem]{Proposition}

\ifarxiv\else
\AtEndEnvironment{proof}{\qed}
\AtEndEnvironment{proofsketch}{\qed}
\fi

\usepackage{tabularx,colortbl}

\makeatletter
\g@addto@macro\bfseries{\boldmath}
\makeatother

\newcommand{\ie}{i.\,e.\xspace}

\newcommand{\eg}{e.\,g.\xspace}

\newcommand{\N}{\mathbb{N}}
\newcommand{\R}{\mathbb{R}}

\DeclareMathOperator{\Unif}{Unif}                         

\newcommand*{\bigO}{\mathcal{O}}

\newcommand{\prob}[1]{\Pr\left(#1\right)}                 
\newcommand{\Prob}[1]{\Pr\left(#1\right)}                 

\newcommand{\ones}[1]{|#1|_1}                             
\newcommand{\zeroes}[1]{|#1|_0}                           

\newcommand{\TRAP}{\textsc{Trap}\xspace}                           

\newcommand{\OMM}{\textsc{OMM}\xspace}                             
\newcommand{\OMMfull}{\textsc{OneMinMax}\xspace}                   

\newcommand{\LOTZ}{\textsc{LOTZ}\xspace}                           
\newcommand{\LOTZfull}{\textsc{LeadingOnesTrailingZeroes}\xspace}  

\newcommand{\OJZJ}{\textsc{OJZJ}\xspace}                           
\newcommand{\OJZJfull}{\textsc{OneJumpZeroJump}\xspace}            

\newcommand{\OTZT}{\textsc{OTZT}\xspace}                           
\newcommand{\OTZTfull}{\textsc{OneTrapZeroTrap}\xspace}            

\newcommand{\mLOTZfull}{\ensuremath{m}\text{-}\textsc{LeadingOnesTrailingZeros}\xspace}

\newcommand{\mOMMfull}{\ensuremath{m}\text{-}\textsc{OneMinMax}\xspace}

\newcommand{\mOJZJfull}{\ensuremath{m}\text{-}\textsc{OneJumpZeroJump}\xspace}

\newcommand{\cwipeout}{c_{\mathrm{wipeout}}}

\newcommand{\nsga}{NSGA\nobreakdash-II\xspace}
\newcommand{\nsgaiii}{NSGA\nobreakdash-III\xspace}

\newcommand{\algnsga}{NSGA\nobreakdash-II\xspace}

\newcommand{\algnsgaiii}{NSGA\nobreakdash-III\xspace}
\newcommand{\algsms}{SMS-EMOA\xspace}
\newcommand{\algspeaii}{SPEA2\xspace}
\newcommand{\algspeaiisim}{SPEA2$^{+}$\xspace}

\newcommand{\inearestdist}[3]{\ensuremath{\sigma^{#1/#2}_{#3}}\xspace}        

\newcommand{\nearestdistleq}[1]{\operatorname{\leq}_{d\left(#1\right)}}       
\newcommand{\nearestdistlt}[1]{\operatorname{<}_{d\left(#1\right)}}           

\newcommand{\nearestdistleqzero}[1]{\operatorname{\leq}_{d_0\left(#1\right)}} 

\ifarxiv\else 
\fi

\newcommand{\toggleplot}[1]{{\textcolor{red}{Plots removed to increase compilation speed. Use command $\backslash$toggleplot in preamble to reinsert them.}}}

\iffinal
\newcommand{\andre}[1]{}
\newcommand{\dirk}[1]{}
\newcommand{\cuong}[1]{}

\newcommand*{\todo}[1]{}
\else
\newcommand{\andre}[1]{\textcolor{green!50!black}{[(Andre) #1]}}
\newcommand{\dirk}[1]{\textcolor{purple}{[(Dirk) #1]}}
\newcommand{\cuong}[1]{\textcolor{orange}{[(Cuong) #1]}}

\newcommand*{\todo}[1]{\textcolor{red}{\textrm{(TODO: #1)}}}
\fi

\ifreview
\usepackage[switch,mathlines]{lineno}

\usepackage{etoolbox} 
\newcommand*\linenomathpatch[1]{
  \cspreto{#1}{\linenomath}%
  \cspreto{#1*}{\linenomath}%
  \csappto{end#1}{\endlinenomath}%
  \csappto{end#1*}{\endlinenomath}%
}
\linenomathpatch{equation}
\linenomathpatch{gather}
\linenomathpatch{multline}
\linenomathpatch{align}
\linenomathpatch{alignat}
\linenomathpatch{flalign}
\fi

\begin{document}

\ifarxiv
\author[]{Duc-Cuong~Dang}
\author[]{Andre~Opris}
\author[]{Dirk~Sudholt}

\affil[]{University of Passau, Passau, Germany}
\date{}
\fi

\title{SPEA2$^+$: Improved Density Estimation in SPEA2 with Provable Runtime~Guarantees}

\maketitle

\begin{abstract}
%
The Strength Pareto Evolutionary Algorithm 2 (SPEA2) is a popular
and prominent evolutionary algorithm for solving multi-objective optimisation 
problems. Despite its popularity, theoretical 
analyses of SPEA2 have only appeared recently. 
Moreover, 
these analyses focus exclusively on how SPEA2 handles non-dominated solutions and 
disregard the algorithmic components responsible for handling dominated solutions.
We conduct a first runtime analysis of SPEA2 for which these components are analysed. 
We prove that, unlike other prominent algorithms, including NSGA-II, NSGA-III
and \algsms under the same setting of constant population size and duplicate elimination, SPEA2 is unable to cover the Pareto
front of the \OTZTfull benchmark efficiently. 
Our results indicate that using $k$-th nearest-neighbour distance in the fitness assignment provides an insufficient signal to maintain diversity among dominated individuals.
To address this issue, we propose an improved 
variant,
\algspeaiisim, that considers all pairwise distances.
The new algorithm achieves the same performance guarantees as the other prominent algorithms 
on \OTZTfull, while matching the performance of the original SPEA2 on simpler problems. 
Experimental results complement our theoretical findings.
\end{abstract}

\section{Introduction}\label{sec:intro}
In everyday life, one often has to explore multiple alternatives and compromises 
before making a decision. A problem of this kind can often be formulated
as a Multi-Objective Optimisation (MOO) Problem. MOO is the area where
Evolutionary Algorithms shine and find many applications. Over the years, many
Multi-Objective Evolutionary Algorithms (MOEAs), such as \nsga~\cite{Deb2002},
\nsgaiii~\cite{Deb2014}, \algsms~\cite{Beume2007}, and \algspeaii~\cite{SPEA2original} 
have emerged as popular and prominent approaches for solving MOO problems.

The versatility of the above-mentioned MOEAs is often attributed to their ability 
to maintain a population of not only high-quality but also diverse solutions. 
This is achieved through a clever combination of the use of dominance 
relations and a density estimation, as discussed~\cite{SPEA2original} and 
rigorously explored in~\cite{Lessons}. For example, the famous algorithm~\algnsga~\cite{Deb2002} 
combines the non-dominated sorting algorithm with the density estimation 
via crowding distances.

Recently, theoretical research on these prominent algorithms, particularly using
runtime analysis, has become a highly active research area, beginning with 
the seminal work of~\cite{ZhengLuiDoerrAAAI22} on \algnsga. Published papers
in this area have tremendously contributed to our understanding of the practical 
success of these algorithms, of their components and their theoretical properties. 
Still, there is much to learn, and in this paper we are interested in analysing 
\algspeaii (Strength-Pareto Evolutionary Algorithm~2)~\cite{SPEA2original}. 
There have been only a few papers~\cite{SPEA22024,DoerrKS2026} 
analysing~\algspeaii. Moreover, the fitness assignment in \algspeaii, which comprises the characteristic Pareto strength calculation and $k$-nearest neighbour density estimation, 
has never been used 
in these analyses. 

\smallskip%
\noindent\textbf{Our contributions:}\label{sec:contrib}
We analyse \algspeaii with a focus on its handling of dominated solutions via fitness assignment based on Pareto strength and $k$-th nearest-neighbour density estimation. We show that, for any constant population size, \algspeaii requires exponential expected time to cover the Pareto front of \OTZTfull. This contrasts with the efficiency of \algnsga, \algnsgaiii, and \algsms, analysed in~\cite{Dang2024Illustrating,DangOS2025}, which rely on non-dominated sorting and alternative diversity mechanisms. \algspeaii is inefficient as the $k$-th nearest-neighbour distance fails to provide a sufficient signal for maintaining diversity within the dominated regions of the objective space.

To address this issue, we propose a refined variant, \algspeaiisim, which replaces the $k$-th nearest-neighbour density estimate with a measure based on all pairwise distances, analogous to the truncation operator used for non-dominated solutions in \algspeaii. Pareto strength remains the primary criterion for dominated solutions. We prove that \algspeaiisim optimises \OTZTfull in polynomial expected time with a constant population size, matching the performance guarantees of the aforementioned algorithms.


To the best of our knowledge, this is the first runtime analysis that explicitly accounts for the fitness assignment mechanism in \algspeaii, which includes the characteristic Pareto strength calculation. Previous analyses either assume sufficiently large populations such that only non-dominated solutions drive progress~\cite{SPEA22024,DoerrKS2026}, or assume that all solutions are non-dominated~\cite{ApproximationSPEA2}, thereby effectively ignoring this component. Consequently, all previously established results for \algspeaii carry over to \algspeaiisim, as our modification only affects the handling of dominated solutions.


From a methodological perspective, we provide two additional insights. First, we show that in the original SPEA2, a dominated solution may be eliminated when $k$ offspring are generated in its vicinity, forming a dense cluster with small $k$-th nearest-neighbour distances and the whole cluster is subsequently removed due to overpopulation (cf.\ Lemma~\ref{lem:elimination}). Second, we establish a geometric property of Euclidean distances in bi-objective spaces: on every bi-objective problem and every population of distinct fitness vectors, extreme points are preferred over their nearest neighbours (cf.\ Lemma~\ref{lem:nearestdistlt-property}). This observation may be of independent interest for analysing \algspeaii and \algspeaiisim.


Finally, experiments complement our theoretical findings. \algspeaiisim, along with other prominent algorithms \algnsga, \algnsgaiii, and \algsms, outperform \algspeaii on \OTZTfull as predicted. For $n=80$ and $\mu=\lambda\leq 25$, the success rate of \algspeaii in covering the Pareto front is less than $10\%$ for a computational budget of $n^3$ while all other algorithms achieve a rate of $100\%$.
%
%
\ifreview\else 
A minor technical contribution is that we found and reported a bug in the current implementation of the truncation operator of SPEA2 in the popular library PyMOO~\cite{Deb2020}. 
\fi
%
%

\smallskip%
\noindent\textbf{Related work:}\label{sec:related-work}
The theoretical study of MOEAs began around 20 years ago, initially focusing on relatively simple algorithms such as (G)SEMO~\cite{LaumannsTZWD02,Laumanns2004,Giel2003,Thierens03,DoerrKV13,Doerr2021}. This algorithm creates the next population solely based on dominance relations. The most widely studied MOEA (with around 60,000 citations) is \algnsga, which uses dominance as the primary criterion for survival selection and the crowding distance as a secondary one. This algorithm performs very well on bi-objective problems, as demonstrated by numerous empirical and theoretical results. The first runtime analysis was conducted by~\citep{ZhengLuiDoerrAAAI22} on $2$-\LOTZ and $2$-\OMM, followed by results on the multimodal problem $2$-\OJZJ~\citep{Qu2022PPSN}, the usefulness of crossover~\citep{Dang2024,DoerrQ23b}, noisy environments~\citep{Dang2023a}, approximation guarantees for covering the Pareto front~\citep{Zheng2022}, lower bounds~\citep{DoerrQu2022}, trap functions~\citep{Dang2024Illustrating}, and stochastic population update~\citep{Bian2023}. It has also been shown that \algnsga can outperform GSEMO exponentially, which relies solely on dominance relations~\citep{Lessons}. In addition, there are results for combinatorial optimisation problems such as the minimum spanning tree problem~\citep{Cerf2023} and the subset selection problem~\citep{MOEASubset}. 

Runtime analyses of other popular MOEAs on simple benchmark functions have only emerged in recent years. These include \algsms~\citep{Zheng_Doerr_2024,ijcai2025p988}, SPEA2~\citep{SPEA22024,DoerrKS2026}, variants of \algnsga~\citep{Krejca2025b}, \algnsgaiii~\citep{WiethegerD23,Opris2024,NearTight2024,OPRIS2026,OprisMultimodal,Opris26PopDyn}, and PAES-25~\cite{Opris2025PAES}. 
There are also several results on approximating the Pareto front when its size exceeds the population size, for example, for 
\algnsga~\citep{Zheng2022}, \algnsgaiii~\citep{ApproximationNSGAIII}, and SPEA2~\citep{ApproximationSPEA2}. Overall, the study of MOEAs is a highly active research area, with even the behaviour of simple algorithms such as GSEMO still being researched~\cite{doerr2025tightruntimeGSEMO}. However, the working principles of \algnsga and \algnsgaiii differ significantly from those of \algspeaii when handling dominated layers. The treatment of the non-dominated layer in \algspeaii is relatively well understood due to its truncation operator, which preserves non-dominated solutions~\cite{SPEA22024} and ensures a good spread along the Pareto front~\cite{DoerrKS2026}. 
In contrast, its behaviour on dominated layers remains largely unexplored once the density estimation comes into play. Currently, there are essentially no rigorous results on whether and how promising solutions can be identified in dominated layers or whether this operator is suitable for optimisation at all.

\section{Preliminaries}\label{sec:prelim} 

%
%
For $n \in \N$, define $[n] := \{1,\dots,n\}$ and
$[n]_0 \coloneqq [n] \cup \{0\}$.
The natural logarithm and that of base-2 are denoted $\ln(\cdot)$ and $\log(\cdot)$ respectively.
%
%
For a bit string $x:=(x_1,\dots,x_n)\in\{0,1\}^n$, we use $\ones{x}$ to denote
its number of $1$-bits, \ie $\ones{x}=\sum\nolimits_{i=1}^{n}x_i$,
and $\zeroes{x}$ to denotes the number of zeroes,
\ie $\zeroes{x}
=n-\ones{x}$.
For two points $X,Y$ with coordinates $X=(x_1,x_2)$ and $Y=(y_1,y_2)$ on the plane, $X$ is said to be weakly to the right of $Y$ if $x_1\geq y_1$. 
%

For a $d$-objective function $f\colon\{0,1\}^n\rightarrow \R^{d}$, 
and for $x, y \in \{0, 1\}^n$, we say $x$ \emph{weakly dominates} $y$
written as $x \succeq y$ (or $y \preceq x$) if $f_i(x) \geq f_i(y)$ for all $i\in[d]$.
$x$ \emph{dominates} $y$ written as $x \succ y$ (or $y \prec x$)
if one inequality is strict. 
Both the weakly-dominance and dominance relations are \emph{transitive},
\eg $x \succ y \wedge y \succ z$ implies $x \succ z$. 
A set of points which covers all possible fitness values not
dominated by any other points in $f$ is called a \emph{Pareto-optimal set} of $f$. 
The image of a Pareto optimal set through $f$ is called the \emph{Pareto front} of $f$. 
A 
solution on the Pareto front is referred to as a \emph{Pareto optimal} solution.

We say an algorithm $\mathcal{A}$ has optimised a function $f\colon\{0,1\}^n\rightarrow \R^{d}$ 
if it has produced (or output) a Pareto-optimal set of $f$. 

\smallskip%
\noindent\textbf{The SPEA2 algorithm:}\label{sec:algo-spea2}
%
\algspeaii~\cite{SPEA2original} is a prominent MOEA and is summarised in 
Algorithm~\ref{alg:SPEA2} for the use of the standard bit mutation on $\{0,1\}^n$
and uniform parent selection.
The algorithm employs an archive $A_t$ alongside a regular population $P_t$. In the
standard notation of the ($\mu$+$\lambda$)~EA scheme, $A_t$ corresponds to the parent population ($|A_t|=\mu$) and $P_t$ corresponds to the offspring population ($|P_t|=\lambda$). 
In each generation, $\lambda$ offspring individuals are
independently created by copying and varying parents selected from $A_t$, and then 
the $\mu$ best individuals of $A_t \cup P_t$ survive to the next generation 
and become the next archive $A_{t+1}$. This survival selection is elitist 
(the plus in $\mu$+$\lambda$) in the sense that non-dominated solutions are 
definitely preferred over dominated ones. To accomplish this, \algspeaii 
first assigns fitness values to solutions, and then employs 
two different operators based on the number  of non-dominated solutions $A'_t$ 
found in the union $A_t \cup P_t$. The chosen operator turns $A'_t$ into 
a population of $\mu$ individuals, which becomes the new archive $A_{t+1}$.

\begin{algorithm2e}[ht] 
\DontPrintSemicolon
\SetKw{KwTo}{to}    
\SetKw{KwOr}{or}    
Initialise $A_0\sim\Unif((\{0,1\}^n)^{\mu})$ and $P_0=\emptyset$; if duplicate elimination is enabled, then use sampling without replacement\;
\For{$t=0$ \KwTo $\infty$}{
    Compute fitness $f(x)=R(x)+D(x)$ for each $x\in A_t \cup P_t$ (see Eq.~\eqref{eq:spea2-fitness})\;
    Let $A_t'$ be the non-dominated solutions (those with $R(x)=0$) from $A_t \cup P_t$\;
    \lIf{$|A_t'| \geq \mu$}{%
    reduce $A_t'$ by applying the truncation operator
    }
    \lElse{%
    expand $A'_t$ 
    by applying the completion operator based on fitness
    }
    Let $A_{t+1}=A_t'$\;
    Let $P_{t+1} = \emptyset$\; 
    \For{$i=1$ \KwTo $\lambda$}{
    Select a solution $x$ from $A_{t+1}$ uniformly at random\;
    Generate $x'$ by flipping each bit of $x$ independently with prob. $1/n$\;    
    \lIf{\text{duplicate elimination is disabled} \KwOr $x'\notin (A_{t+1} \cup P_{t+1})$}{$P_{t+1} = P_{t+1} \cup \{x'\}$}
    }\label{alg:SPEA2:line:add-offspring}
    }
\caption{SPEA2~\cite{SPEA2original} with standard bit mutation on $\{0,1\}^n$}
\label{alg:SPEA2}
\end{algorithm2e}

If $|A'_t|\geq \mu$, a truncation operator is applied to reduce the size of
$A'_t$ to~$\mu$. The operator removes individuals one by one, starting with the
one with the smallest distance to its nearest neighbour in the fitness space,
and ties are broken with the distances to the second, then the third and so on
nearest neighbours. Formally, let $\inearestdist{i}{P}{x}$ be the Euclidean 
distance in the objective space of $x$ to its $i$-th nearest neighbour from a 
population $P$ that contains $x$, and 
consider the following relations $\nearestdistlt{P}$ and $\nearestdistleq{P}$
on the elements of $P$: 
\begin{align}
x \nearestdistlt{P} y 
    &\Leftrightarrow  
    \left(\exists 0<\ell<|P|\colon (\forall 0<i<\ell: \inearestdist{i}{P}{x} = \inearestdist{i}{P}{y}) \wedge \inearestdist{\ell}{P}{x} < \inearestdist{\ell}{P}{y}\right),
    \label{eq:spea2-nearestdistlt}\\
x \nearestdistleq{P} y 
    &\Leftrightarrow 
    \left(x \nearestdistlt{P} y\right) \vee \left(\forall 0<i<|P|\colon \inearestdist{i}{P}{x} = \inearestdist{i}{P}{y} \right).
    \label{eq:spea2-nearestdistleq}
\end{align}
The operator then repeatedly removes and updates the smallest individual 
based on $\nearestdistleq{A'_t}$, breaking ties arbitrarily or uniformly 
at random, until $|A'_t| = \mu$. The strict order $\nearestdistlt{P}$ will 
be useful later for our analysis as follows. For some $y\in A'_t$, if one can show 
that there always exists a $x\in A'_t\setminus\{y\}$ with $x\nearestdistlt{P} y$
during the iterative removal-update process, then $y$ is guaranteed to remain in $A'_t$. 


Otherwise, if $|A'_t|<\mu$, a completion operator based on fitness is applied 
to fill in $A'_t$ with $\mu-|A'_t|$ dominated individuals from $(A_t \cup P_t)
\setminus A'_t$. A fitness $f(x)$ is assigned for each individual $x$ of $A_t
\cup P_t$, and it contains two components:
\begin{align}
f(x)&:=R(x) + D(x),\label{eq:spea2-fitness}\\
\text{where } R(x)&:=\sum_{y\in P_t\cup A_t\colon y\succ x}S(y) \text{ with } S(x):=|\{y \in A_t\cup P_t\mid y \succ x\}|,\label{eq:spea2-raw-fitness}\\
\text{and } D(x)&:=1/(2+\inearestdist{k}{A_t \cup P_t}{x}) \text{ with } k:=\lfloor \sqrt{\mu+\lambda}\rfloor.\label{eq:spea2-kth-nearest-density}
\end{align}
The first component $R(x)$, called the \emph{raw fitness}, is based on the
notion of \emph{Pareto strength} $S$ of solutions. That is, $S(x)$ is the number
of solutions that $x$ dominates and $R(x)$ is the sum of strengths of the
solutions that dominate $x$. The second component $D(x)$, called the \emph{density
estimate} of $x$, uses the distance to the $k$-th nearest neighbour in the objective 
space 
to estimate
how crowded the area surrounding $f(x)$ is in $f(A_t\cup P_t)$. The formulation above
ensures that $D(x)\in (0,1)$ while $R(x)$ are integers, thus $D(x)$ is a secondary 
selection criterion for solutions with identical raw fitness. 
The solutions of $(A_t \cup P_t) \setminus A'_t$ are sorted in an ascending order 
of fitness, and the $\mu-|A'_t|$ first solutions are added to $A'_t$.

Our description of \algspeaii deviates slightly from the original algorithm in the initialisation phase. Specifically, in the original paper, $P_0$ is initialised uniformly at random, while $A_0$ is set to be empty. However, the subsequent survival selection transfers all individuals from $P_0$ to $A_1$ in any case. In contrast, in our description, $A_0$ is initialised uniformly at random and then carried over to $A_1$. Therefore, the two formulations are in fact equivalent.
 
Unlike \algnsga where solutions are ranked by the layer of non-dominated solutions 
that they belong to, and further by a density estimate using crowding distances, 
\algspeaii ranks its solutions by the Pareto strength of those that dominate 
them and then by a density estimate based on the distance to the nearest 
neighbours in the objective space. 
There is an asymmetry in the use of density estimations of \algspeaii, that 
is when $A'_t\geq \mu$ all the distances are used in the density estimation, but 
when $A'_t < \mu$, only the $k$-th smallest distance is used.

\smallskip%
\noindent\textbf{An improved algorithm, \algspeaiisim:}
We consider a more symmetric variant of the completion procedure as follows. 
The solutions of $(A_t \cup P_t) \setminus A'_t$ are removed one-by-one until the 
(multi-)set size reaches $\mu - |A'_t|$, after which the remaining solutions are added to $A'_t$.
The removal prioritises the solution $x$ with the smallest reciprocal 
raw fitness $1/R(x)$, and ties are broken according to the distances to the first, 
then second and so on, nearest neighbours. This is similar to the truncation
operator, except that we introduce a primary criterion based on $R(x)$.
Formally, we define $\inearestdist{0}{P}{x}:=1/R(x)$ and augment the definition
of $\nearestdistlt{P}$ to $\nearestdistleqzero{P}$ by allowing the index 0 for 
$i$ and $\ell$ in Eqs.~\eqref{eq:spea2-nearestdistlt} \eqref{eq:spea2-nearestdistleq}. 
The completion procedure
then repeatedly removes the smallest individual of $(A_t \cup P_t) \setminus A'_t$
with respect to $\nearestdistleqzero{(A_t \cup P_t) \setminus A'_t}$ until its size is 
$\mu - |A'_t|$, and subsequently merges the remaining individuals into $A'_t$.
The variant of \algspeaii that employs this completion procedure instead of the
original one is referred to as \algspeaiisim. 

\smallskip%
\noindent\textbf{OneTrapZeroTrap:}\label{sec:otzt-function}
%
The \OTZTfull function (\OTZT for short) for $x\in \{0,1\}^n$ is:
\begin{align*}
\OTZT(x) := \left(\sum_{i=1}^{n}x_i + (n+1)\prod_{i=1}^{n}(1-x_i),
                  \sum_{i=1}^{n}(1-x_i) + (n+1)\prod_{i=1}^{n}x_i\right).
\end{align*}

The function has the same fitness vector as the \OMMfull function (which 
ignores the product parts in the above formula) for all points in $\{0,1\}^n$ 
except for the two points $1^n$ and $0^n$ (where the product parts become 
relevant). Particularly, $f(1^n)=(n,n+1)$ and $f(0^n)=(n+1,n)$, thus these 
two points dominate all other search points of $\{0,1\}^n$ but are incomparable 
with each other and hence Pareto optimal.
The function name comes from the fact that the first objective 
$\OTZT_1(x)$ is the well-known \TRAP function of the input string, while 
the second objective $\OTZT_2(x)$ is that of the complement string to the input.


\smallskip%
\noindent\textbf{Why duplicate elimination is essential:}\label{sec:remove-genotype-dup}
%
%
It has been shown in~\cite{Dang2024Illustrating,DangOS2025} that without
the diversity mechanism of (genotype) duplicate elimination when generating 
offspring populations, any algorithm that employs the ($\mu$+$\mu$) survival 
selection scheme with $\mu=o(\sqrt{n})$ requires exponential time to optimise 
the \OTZTfull function. The reason is that once a Pareto optimal solution is 
found, it can easily be cloned, and its clones will take over the population. 
This result can be generalised to the ($\mu$+$\lambda$) selection of 
the vanilla (default version of) \algspeaii and \algspeaiisim. 

\begin{toappendix}
This appendix contains full proofs for all statements whose proof was omitted or only sketched in the main part, so that the reviewers may check the correctness of our results at their discretion. The appendix will not be part of the published paper, should it be accepted, but a full version will be made available on a preprint server.
\end{toappendix}

\begin{theoremrep}[Generalised from Theorem~4.4 in~\cite{DangOS2025}]\label{thm:require-duplicate-elimination}
The vanilla \algspeaii and \algspeaiisim (without genotype duplicate elimination)
using uniform parent selection, standard bit mutation, and $\mu$,$\lambda \in o(\sqrt{n})$ 
require $\Omega(n^{n})$ fitness evaluations in expectation to cover the Pareto 
front of \OTZTfull.
\end{theoremrep}
\begin{appendixproof}
The arguments below only rely on the $(\mu+\lambda)$ selection scheme, thus
they hold for both algorithms, thus we can assume \algspeaii. 
Let $F:=\{0^n, 1^n\}$ be the unique Pareto optimal set of \OTZT. By a union 
bound the probability that the initial population does not find this set directly
is at least $1- 2\cdot \mu\cdot  2^{-n}=1-o(1)$.  

Assuming the above event happens, we first estimate the probability that both optima 
can be found in one generation $t$ of the algorithm, conditioned on one being found, 
by looking at its complement event $E_1$.
We optimistically assume that when a solution is selected, the probability
generating an individual in $F$ is $2/n$, as at least one bit needs to be flipped
and the factor $2$ here is from the union bound on $|F|$. 
Let $X$ be the number of solutions of $F$ in $P_{t+1}$, then the probability
of discovering $F$ is
\[
p_1 
    = \prob{X\geq 1} 
    = 1 - \left(1-\frac{2}{n}\right)^{\lambda}
    \geq \frac{2\lambda/n}{2\lambda/n+1}
    = \frac{2\lambda/n}{o(1)+1} = \Omega\left(\frac{\lambda}{n}\right),
\]
while that of fully covering the front is at most, here using Bernoulli's inequality,
\begin{align*}
p_2 &= \prob{X\geq 2}
     = 1 - \prob{X=0} - \prob{X=1}\\
    &= 1 - \left(1-\frac{2}{n}\right)^{\lambda} - {\lambda \choose 1}\left(\frac{2}{n}\right)\left(1-\frac{2}{n}\right)^{\lambda-1}\\
    &= 1 - \left(1-\frac{2}{n}+\frac{2\lambda}{n}\right)\left(1-\frac{2}{n}\right)^{\lambda-1}\\
    &\leq 1 - \left(1+\frac{2(\lambda-1)}{n}\right)\left(1-\frac{2(\lambda-1)}{n}\right)
    = \frac{4(\lambda-1)^2}{n^2} = O\left(\frac{\lambda^2}{n^2}\right)
\end{align*}
Thus $\prob{E_1}\leq p_2/p_1=O(\lambda/n)=o(1)$, so with probability at least 
$1-\prob{E_1} =1-o(1)$, the Pareto front is not covered in the generation in which
the first Pareto-optimal is found. Again, assume this occurs, we now estimate
the probability that the first discovered solution $x\in F$ takes over the population 
before the other solution $y\in F\setminus\{x\}$ is found. An offspring production
is called \emph{good} if $x$ is cloned, and the event is denoted $G$. 
This event occurs with a probability of at least 
$\prob{G}=\frac{1}{\mu}\left(1-\frac{1}{n}\right)^{n}\geq\frac{1}{4\mu}$ 
for $n\geq 2$. The production is bad, \ie the event is denoted $B$, if 
$y$ is created and $\prob{B}\leq \frac{1}{n}$, as at least one bit needs to 
be flipped. For $x$ to take over the population, it suffices that $(\mu-1)+(\lambda-1)=
\mu+\lambda-2$
good events occur before the first bad one. Here, the $\lambda-1$ term takes
into account that the $(\mu-1)$-th good event can occur at the first offspring
production of the $\lambda$ ones in a generation, thus the $\lambda-1$ remaining 
productions still need to be good. The probability of the event is
\begin{align*}
\left(\frac{\prob{G}}{\prob{G}+\prob{B}}\right)^{\mu+\lambda-2}
  &= \left(\frac{1/(4\mu)}{1/(4\mu) + 1/n}\right)^{\mu+\lambda-2}\\
  &= \left(1-\frac{1/n}{1/(4\mu) + 1/n}\right)^{\mu+\lambda-2}
   \geq \left(1-\frac{4\mu}{n}\right)^{\mu+\lambda-2}\\
  &\geq 1 - \frac{4\mu (\mu+\lambda-2)}{n} = 1 - o(1).
\end{align*}
After the series of events, 
which occur with probability $1-o(1)$, the entire population $A_t$ 
only contains the Pareto-optimal solution $x$. 
It then takes $\Omega(n^{n})$ 
fitness evaluations in expectation to create the other Pareto-optimal solution 
$y$ as all bits must be flipped. This is also the expected running time of 
the algorithm.
\end{appendixproof}

In the following, we therefore consider algorithms with genotype 
duplicates being removed at initialisation and during the run.

\section{Drawback of the Density Estimation in SPEA2}\label{sec:density-estimation}

In this section, we show that SPEA2 with duplicate elimination and density estimation cannot cover the Pareto front of \OTZTfull with high probability. The issue arises when one Pareto-optimal search point has already been found (say $1^n$), while the other has not. In this situation, it becomes difficult to maintain solutions far away from $1^n$ as there is no useful signal to diversify the dominated individuals. More precisely, we show that, with constant probability, within a constant number of generations, all individuals with at least $n/2$ zeros are removed. We refer to this as a \emph{wipeout} as it wipes out much progress towards reaching the missing Pareto-optimal search point, $0^n$. 

The key idea of the proof is that a search point $x^*$ with at least $n/2$ zeros can be removed when mutation creates many offspring close to~$x^*$, forming a cluster of exactly $k+1$ search points with a small $k$-th nearest neighbour distance that can then be removed all at once in the selection step. In other words, a whole cluster of individuals may suddenly die from overpopulation. To our knowledge, such an effect has only been proven once before, for the fitness sharing diversity mechanism~\cite{Oliveto2019}.



\begin{lemmarep}
\label{lem:elimination}
Consider SPEA2 with $\mu,\lambda \in O(1)$, $\lambda \ge k$ and duplicate elimination on \OTZTfull at some time~$t_0$. Suppose that $1^n \in A_{t_0}$, but $0^n \notin A_{t_0}$. Define $B_t := \{x \in A_t \mid \zeroes{x} \geq n/2\}$. 
%
There is a constant ${\cwipeout > 0}$ (depending on $\mu, \lambda)$ such that, with probability at least $\cwipeout$, we have $B_{t^*} = \emptyset$ for some $t^* \in [t_0, t_0 + (\lceil{\mu/\lambda}\rceil+1)|B_{t_0}|]$, or $0^n$ is discovered during this time span.
\end{lemmarep}

The condition $\lambda \ge k$ is necessary for creating a cluster of $k+1$ similar search points. It is implied by $\mu \le \lambda(\lambda-1)$, which holds, \eg, for $\mu=\lambda \ge 2$.
To prove Lemma~\ref{lem:elimination}, we introduce a notion of isolated individuals whose $k$-nearest neighbour distance is guaranteed to be at least~$2$.
\begin{definition}
\label{def:isolated}
    We call an individual $x \in P$, $x \notin \{0^n, 1^n\}$, isolated in a population $P$ if no search point $y \in P \setminus \{x\}$ satisfies $f_1(y) \in \{f_1(x)-1, f_1(x), f_1(x)+1\}$ or $f_2(y) \in \{f_2(x) - 1, f_2(x), f_2(x) + 1\}$.
\end{definition}

Now we show that isolated individuals have a larger $k$-th nearest neighbour distance than individuals in a cluster of search points with $i-1$ or $i$ zeros.
\begin{lemmarep}
\label{lem:properties-of-distances}
    Considering the function \OTZTfull, for every population $P$, the following statements hold.
    \begin{enumerate}
        \item If $x \in P$ is isolated in~$P$, then\footnote{We remark that the converse is not true: $x$ can have a large value of $\inearestdist{k}{P}{x}$, but the closest $k-1$ search points can be arbitrarily close or even identical to~$x$.} $\inearestdist{k}{P}{x} \ge 2$.
        \item $\inearestdist{k}{P}{x} = 0$ iff $P$ contains at least $k$ search points with fitness vector $f(x)$.
        \item For any $i \in \{2, \dots, n-1\}$, if $P$ contains $k+1$ search points $y_1, \dots, y_{k+1}$ with $|y_j|_0 \in \{i-1, i\}$ for all $j \in [k+1]$ then $\inearestdist{k}{P}{y_j} \le \sqrt{2}$ for all $j \in [k+1]$.
    \end{enumerate}
\end{lemmarep}
\begin{appendixproof}
For the first statement, both objective values of an individual nearest to an isolated one differ from those of the isolated individual by at least 2. Hence, the nearest distance is at least $\inearestdist{1}{P}{x} \ge \sqrt{2+2} = 2$ and since $\inearestdist{k}{P}{x} \ge \inearestdist{1}{P}{x}$, the same lower bound also holds for the $k$-th nearest distance.

Regarding the second statement, the ``if'' part is obvious. If $\inearestdist{k}{P}{x} = 0$ then the $k$ nearest individuals must have Euclidean distance 0 to~$x$, implying that their fitness vector must be~$f(x)$.

The third statement follows since, in each objective, the $k$ closest objective vectors only differ by at most~1 from that of~$x$. Hence, the Euclidean distance to the $k$-th nearest neighbour is at most $\sqrt{1+1} = \sqrt{2}$.
\end{appendixproof}

\ifarxiv
We now use Lemma~\ref{lem:properties-of-distances} to prove Lemma~\ref{lem:elimination}.
\fi

\begin{proofsketch}[of Lemma~\ref{lem:elimination}]
We first show the following statement: There is a constant $c > 0$ (depending on $\mu, \lambda$) such that, with probability at least $c$, we have $|B_{t'}| < |B_{t_0}|$ for some $t' \in [t_0, t_0 + \lceil \mu/\lambda \rceil+1]$. 

All individuals distinct from $0^n$ and $1^n$ are dominated by~$1^n$. Hence, as long as no search point $0^n$ is created, the raw fitness value of all dominated solutions is the same. So, only the density estimate determines the survival of solutions. Since the density estimate $D(x)$ is strictly decreasing in the $k$-th nearest neighbour distance $\inearestdist{k}{A_t \cup P_t}{x}$, $A_{t+1}$ will contain $1^n$ and the $\mu-1$ individuals from $A_t \cup P_t$ with the largest $k$-th nearest neighbour distance.

We apply the method of typical runs~\cite[Section~5.6]{Jansen2013} and define a sequence of phases, each with a specific goal. For each phase, we estimate the probability of achieving its respective goal by stating and analysing sufficient conditions for success. A phase begins once the goal of the preceding phase has been achieved. The goal of the final phase implies the desired statement. 
We use $Z_i \coloneqq \{x \mid \zeroes{x} = i\}$ to denote the set of search points with $i$ zeros.

\noindent\textbf{Phase~1:} The goal is to evolve a population $A_{t_1}$ in which $B_{t_1} \subseteq B_{t_0}$ and every $x \in A_{t_1}$ satisfies $\inearestdist{k}{A_{t_1}}{x} \ge 2$.

This can be achieved by creating isolated search points. To this end, we repeatedly select $1^n$ as a parent and flip a suitable number of bits to create an offspring with a unique fitness vector satisfying Definition~\ref{def:isolated}. This can be done by flipping up to $3(\mu+\lambda) = O(1)$ many arbitrary bits in $1^n$. If the next $\mu$ offspring are created in this way, the population consists only of isolated individuals, and the goal follows by Lemma~\ref{lem:elimination}. Since $\mu,\lambda=O(1)$, both selecting $1^n$ and flipping a prescribed constant number of bits occur with constant probability, and the probability of a sequence of $O(1)$ such events is still bounded from below by a constant $d_1>0$.

\noindent\textbf{Phase~2:} The goal is, for some $x^* \in B_{t_1}$, to generate~$P_{t_1}$ by creating up to $k$ offspring in $Z_{\zeroes{x^*}-1}$ such that $A_{t_1} \cup P_{t_1}$ contains exactly ${k+1}$ individuals $y_1, \dots, y_{k+1} \in (Z_{\zeroes{x^*}} \cup Z_{\zeroes{x^*}-1})$ (which includes $x^*$), 
all other offspring in $P_{t_1}$ have at most $3(\mu+\lambda)$ zeros and are isolated in $A_{t_1} \cup P_{t_1}$. 


To achieve this, fix $x^* \in B_{t_1}$ and repeatedly select it as a parent, flipping a single $0$-bit. As duplicate solutions are not accepted, not all such flips succeed; however, at least $\zeroes{x^*}-\mu-\lambda \ge n/2 - O(1)$ bits are suitable, so a mutation succeeds with probability at least $(n/2-O(1))/(en) = \Omega(1)$. The remaining offspring are generated as in Phase~1. Hence, the probability of realising this phase is bounded from below by a constant $d_2>0$.

\noindent\textbf{Phase~3:} The goal is for selection to create a population $A_{t_1+1}$ with ${B_{t_1+1} \subsetneq B_{t_0}}$, thus $|B_{t_1+1}| < |B_{t_0}|$.

If Phase~2 succeeds, the $k+1$ individuals $y_1,\dots,y_{k+1}$ have $k$-th nearest-neighbour distance at most $\sqrt{2}$ by Lemma~\ref{lem:properties-of-distances}, and are therefore among the worst-ranked individuals.  If there are more than $\lambda$ individuals of $k$-th nearest neighbour distance at most~$\sqrt{2}$, we may need to rely on uniform tie breaking selecting $y_1, \dots, y_{k+1}$ for removal. Since $\mu, \lambda = O(1)$, the probability of creating an appropriate $A_{t_1+1}$ is bounded from below by a constant~$d_3 > 0$.

If Phase~3 succeeds, $B_{t_1+1} \subseteq B_{t_0} \setminus \{x^*\}$, which implies $B_{t_1+1} \subsetneq B_{t_0}$. 
Together, this shows the statement from the start of the proof for the constant probability bound of $c \coloneqq d_1 \cdot d_2 \cdot d_3 > 0$.

The statement of the lemma now follows from applying this statement repeatedly, up to $|B_{t_0}| \le \mu$ times, noting that these events are independent because they consider disjoint events over time. The probability bound is at least $c^{|B_{t_0}|} \ge c^{\mu}$, and we can take $\cwipeout \coloneqq c^\mu = \Omega(1)$ since $\mu = O(1)$.
\end{proofsketch}

\begin{appendixproof}[Proof of Lemma~\ref{lem:elimination}]
We ignore the possibility that the statement may be fulfilled by creating $0^n$ (we shall argue later that this is unlikely). We first show the following statement: There is a constant $c > 0$ (depending on $\mu, \lambda$) such that, with probability at least $c$, we have $|B_{t'}| < |B_{t_0}|$ for some $t' \in [t_0, t_0 + \lceil \mu/\lambda \rceil+1]$. 

All individuals distinct from $0^n$ and $1^n$ are dominated by~$1^n$. Hence, as long as no search point $0^n$ is created, the raw fitness value of all dominated solutions is the same. So, only the density estimate determines the survival of solutions. Since the density estimate $D(x)$ is strictly decreasing in the $k$-th nearest neighbour distance $\inearestdist{k}{A_t \cup P_t}{x}$, $A_{t+1}$ will contain $1^n$ and the $\mu-1$ individuals from $A_t \cup P_t$ with the largest $k$-th nearest neighbour distance.

We apply the method of typical runs~\cite[Section~5.6]{Jansen2013} and define a sequence of phases, each with a specific goal. For each phase, we estimate the probability of achieving its respective goal by stating and analysing sufficient conditions for success. A phase begins once the goal of the preceding phase has been achieved. The goal of the final phase implies the desired statement. We pessimistically ignore the possibility that the final goal may already be achieved during an earlier phase. We use $Z_i \coloneqq \{x \mid \zeroes{x} = i\}$ to denote the set of search points with $i$ zeros.

\noindent\textbf{Phase~1:} The goal is to evolve a population $A_{t_1}$ in which $B_{t_1} \subseteq B_{t_0}$ and every $x \in A_{t_1}$ satisfies $\inearestdist{k}{A_{t_1}}{x} \ge 2$.


\noindent\textbf{Phase~2:} The goal is, for some $x^* \in B_{t_1}$, to generate~$P_{t_1}$ by creating up to $k$ offspring in $Z_{\zeroes{x^*}-1}$ such that $A_{t_1} \cup P_{t_1}$ contains exactly ${k+1}$ individuals $y_1, \dots, y_{k+1} \in (Z_{\zeroes{x^*}} \cup Z_{\zeroes{x^*}-1})$ (which includes $x^*$), 
all other offspring in $P_{t_1}$ have at most $3(\mu+\lambda)$ zeros and are isolated in $A_{t_1} \cup P_{t_1}$.


\noindent\textbf{Phase~3:} The goal is for selection to create a population $A_{t_1+1}$ with ${B_{t_1+1} \subsetneq B_{t_0}}$, thus $|B_{t_1+1}| < |B_{t_0}|$.


\smallskip%
\noindent\textbf{Analysing Phase~1:} A sufficient condition for reaching this goal is the intersection of the following events. For $t = t_0 +1$ up to $t = t_0 + \lceil \mu/\lambda\rceil$, the algorithm only generates offspring that are isolated in $A_t \cup P_t$ and have less than $n/2$ zeros. 
By Lemma~\ref{lem:properties-of-distances}, isolated individuals have a $k$-th nearest neighbour distance of at least~2. Let $\varphi(A_t)$ denote the number of search points with $k$\nobreakdash-th nearest neighbour distance of at least~2 in $A_t$. Since selection is performed according to the largest $k$-th nearest neighbour distance, if the stated events occur, $\varphi(A_{t+1}) = \min \{\varphi(A_t) + \lambda, \mu\}$. Then, at some time $t_1 \le t + \lceil \mu/\lambda\rceil$, we will have $\varphi(A_{t_1}) = \mu$ and $B_{t_1} \subseteq B_{t_0}$.

We estimate the probability of the above events. Let us call a value $v \in [n-1]$ \emph{unoccupied} if an offspring with $v$ zeros would be isolated in the set containing $A_t$ and all offspring created in generation~$t$ so far. For every offspring created in generation~$t$, at most three fitness vectors become occupied. Hence, at most $3(\mu+\lambda-1)$ fitness vectors from $V:=\{2, \ldots , 3(\mu+\lambda)\}$ are occupied, and there exists a unoccupied value $v \in V$. A sufficient condition to generate an offspring with $v$ zeros is to choose $1^n$ as a parent, and then to flip $v$ bits, creating $v$ zeros. The probability of choosing $1^n$ as parent and flipping $1 \le v \le 3(\mu+\lambda)$ bits is at least 
\[
    \frac{1}{\mu}\binom{n}{v} \left(\frac{1}{n}\right)^{v} \left(1 - \frac{1}{n}\right)^{n-v} \ge \frac{1}{e\mu \cdot v!} \ge \frac{1}{e\mu \cdot (3(\mu+\lambda))!}
\]
and hence the sought probability is at least 
\begin{align*}
\left(\frac{1}{e\mu \cdot (3(\mu + \lambda))!}\right)^{\lambda \cdot \lceil \mu/\lambda \rceil} \geq d_1
\end{align*}
for a sufficiently small constant $d_1>0$. Note that these offspring have less than $n/2$ zeros since $3(\mu+\lambda) = O(1) < n/2$ for sufficiently large~$n$.

\smallskip%
\noindent\textbf{Analysing Phase~2:} Now fix an arbitrary search point $x^* \in B_{t_1}$ and let $\ell \coloneqq |A_{t_1} \cap  (Z_{\zeroes{x^*}} \cup Z_{\zeroes{x^*}-1})|$. Since Phase~1 has ended successfully, all distances in $A_{t_1}$ are at least~2. This implies $\ell < k+1$ as otherwise the considered search points would have a $k$-th nearest neighbour distance of at most~$\sqrt{2}$ by Lemma~\ref{lem:properties-of-distances}, a contradiction.
To generate an offspring $y \in (Z_{\zeroes{x^*}} \cup Z_{\zeroes{x^*}-1})$, it suffices to choose $x^*$ as a parent, and to flip a single 0-bit to create an individual which is different to those in $A_t$, and those previously generated in that generation to avoid genotypic duplicates. This happens with probability at least 
\begin{align*}
\frac{1}{\mu} \cdot \frac{\zeroes{x^*}-\mu-\lambda}{n} \cdot \left(1-\frac{1}{n}\right)^{n-1} \geq \frac{1}{3e \mu}
\end{align*}
for $n$ sufficiently large. 
Phase~2 is successful if exactly $k-\ell+1$ such offspring are generated (which is feasible since $\lambda \ge k$), and the remaining $\lambda-k+\ell-1$ offspring are isolated in $A_{t_1} \cup P_{t_1}$ and have at most $3(\mu+\lambda)$ zeros. 
Re-using arguments from the analysis of Phase~1, the considered events happen with probability at least
\begin{align*}
\binom{\lambda}{k-\ell} \cdot \left(\frac{1}{3e \mu}\right)^{k-\ell+1} \cdot \left(\frac{1}{e\mu \cdot (3(\mu + \lambda))!} \right)^{\lambda-k+\ell-1} \geq d_2
\end{align*}
for a sufficiently small positive constant $d_2>0$. 

\smallskip%
\noindent\textbf{Analysing Phase~3:} if Phase~2 ends successfully, the $k+1$ individuals $y_1, \dots, y_{k+1}$ have a $k$-th nearest neighbour distance of at most $\sqrt{2}$ by Lemma~\ref{lem:properties-of-distances}, while all other individuals in $A_{t_1} \cup P_{t_1}$ are isolated in $(A_{t_1} \cup P_{t_1}) \setminus Z_{\zeroes{x^*}-1}$. Note that individuals\footnote{We cannot rely on selection removing search points in $A_{t_1} \cap Z_{\zeroes{x^*}-2}$ to shrink $B_t$ as we do not know whether they are contained in $B_t$. They are not contained in~$B_t$ if $\zeroes{x^*} \le n/2+1$.} in $A_{t_1} \cap Z_{\zeroes{x^*}-2}$ may no longer be isolated in $A_{t_1} \cup P_{t_1}$ owing to the offspring from $Z_{\zeroes{x^*}-1}$.
However, they will retain a $k$-th smallest distance of at least $\sqrt{2}$ by the second statement of Lemma~\ref{lem:properties-of-distances} since no offspring were added to $Z_{\zeroes{x^*}-2}$.
If $A_{t_1} \cup P_{t_1}$ contains $s$ search points of $k$-th smallest distance at most~$\sqrt{2}$, the probability of choosing $y_1, \dots, y_{k+1}$ for removal is 1 if $s \le \lambda$ and if $s > \lambda$, it is at least
\[
    \binom{\lambda}{k+1}/\binom{s}{\lambda} = \frac{\lambda!\lambda!(s-\lambda)!}{(k+1)!(\lambda-k-1)!s!} \ge d_3
\]
for a constant $d_3 > 0$.

If Phase~3 succeeds, $B_{t_1+1} \subseteq B_{t_0} \setminus \{x^*\}$, which implies $B_{t_1+1} \subsetneq B_{t_0}$. 
Together, this shows the statement from the start of the proof for the constant probability bound of $c \coloneqq d_1 \cdot d_2 \cdot d_3 > 0$.

The statement of the lemma now follows from applying this statement repeatedly, up to $|B_{t_0}| \le \mu$ times, noting that these events are independent because they consider disjoint events over time. The probability bound is at least $c^{|B_{t_0}|} \ge c^{\mu}$, and we can take $\cwipeout \coloneqq c^\mu = \Omega(1)$ since $\mu = O(1)$.
\end{appendixproof}

Now we are in a position to prove the main result from this section, using Lemma~\ref{lem:elimination} to argue that \algspeaii repeatedly experiences wipeouts that destroy progress towards the undiscovered Pareto-optimal search point.
\begin{theoremrep}
Consider SPEA2 with $\mu,\lambda \in O(1)$, $\lambda \ge k$, with or without duplicate elimination on \OTZTfull. The expected number of iterations until SPEA2 covers the entire Pareto front (in particular, until it finds both $1^n$ and $0^n$) is $e^{\Omega(n)}$.
\end{theoremrep}

\begin{proofsketch}
We only need to consider SPEA2 using duplicate elimination because otherwise the result follows from Theorem~\ref{thm:require-duplicate-elimination}. 
Let $t_0$ denote the first iteration in which a first Pareto-optimal search point, $0^n$ or $1^n$, is discovered. By symmetry, we may assume that this is $1^n$. 

Let $B_t$ be as defined in Lemma~\ref{lem:elimination}. We call a generation when $B_t$ is emptied a \emph{wipeout}. A population $A_t$ with $1^n \in A_t$ and $B_t = \emptyset$ is called \emph{safe}. 
By Lemma~\ref{lem:elimination}, along with the fact that creating $0^n$ for the first time has a probability of at most $1/n$, with probability $\Omega(1)$, a wipeout occurs before $0^n$ can be found. 

Starting from a safe population, a necessary condition for finding $0^n$ in the next $T = n/(4\lambda)$ generations without returning to a safe population is that one of the following events happens.
\begin{enumerate}
    \item During $T$ generations, the algorithm evolves a lineage from some search point $x$ containing an individual with a Hamming distance at least $n/2$ to~$x$. Since $\zeroes{x} < n/2$, this is necessary for evolving $0^n$ from a safe population. A necessary event for creating such a lineage is that all standard bit mutations executed in $T$ generations flip at least $n/2$ bits in total. Since the expected number of flipping bits in all $\lambda T$ mutations is $\lambda T = n/4$, the probability of this event is $e^{-\Omega(n)}$ by Chernoff bounds. 
    \item During $T$ generations, no wipeout happens. The probability of this event is estimated as follows. We divide those $T$ generations into $\lfloor T/\tau \rfloor$ periods of $\tau$ generations each. We then apply Lemma~\ref{lem:elimination} at the start of each period. Note that the probability of having a wipeout in any period is at least $\cwipeout$, independently of previous periods. Consequently, the probability of never experiencing a wipeout in $T$ generations is at most $(1-\cwipeout)^{\lfloor T/\tau\rfloor} = e^{-\Omega(\cwipeout \cdot T/\tau)} = e^{-\Omega(T)} = e^{-\Omega(n)}$. 
\end{enumerate}

If none of the above events occurs, the population contains only search points with fewer than $n$ zeros until the next wipeout happens.
A wipeout creates another safe population, allowing us to iterate the above argument. 
Hence, by a union bound over $e^{\varepsilon n}$ such periods, for a suitably small constant $\varepsilon > 0$, $0^n$ is not found in the next $e^{\varepsilon n}$ generations with probability $1 - e^{\varepsilon n} \cdot e^{-\Omega(n)} = 1-e^{-\Omega(n)+\varepsilon n} = 1 - e^{-\Omega(n)}$. This happens with a total probability of $\Omega(1)$, and, by the law of total probability, the expected number of generations to optimise \OTZTfull is $\Omega(1) \cdot e^{\varepsilon n} = e^{\Omega(n)}$.
\end{proofsketch}

\begin{appendixproof}
We only need to consider SPEA2 using duplicate elimination because otherwise the result follows from Theorem~\ref{thm:require-duplicate-elimination}. 
Recall that it generates the initial population using uniform sampling without replacement. The probability that all individuals from $A_0$ do not coincide with $0^n$ and $1^n$ is at least $(1 - 2/(2^n-\mu))^\mu = 1 - e^{-\Omega(n)}$. Suppose that this event, $E_1$, happens. 

Let $t_0$ denote the first iteration in which a first Pareto-optimal search point, $0^n$ or $1^n$, is discovered. By symmetry, we may assume that this is $1^n$. Then, the probability of finding $0^n$ in the same iteration is at most $\lambda/n=o(1)$, since in one iteration it requires flipping one specific bit and $\lambda$ is constant. We assume that $0^n$ is not found in the same iteration, denoted as event $E_2$, and note that then Lemma~\ref{lem:elimination} is in force.

Let $B_t$ be as defined in Lemma~\ref{lem:elimination}. We call a generation when $B_t$ is emptied a \emph{wipeout}. A population $A_t$ with $1^n \in A_t$ and $B_t = \emptyset$ is called \emph{safe}. 
By Lemma~\ref{lem:elimination}, the probability of a wipeout occurring after $(\lceil{\mu/\lambda}\rceil + 1)|B_t| \leq (\lceil{\mu/\lambda}\rceil + 1)\mu \eqqcolon \tau$ generations is at least some constant $\cwipeout > 0$.
If the population obtained after generating $1^n$ is not safe, a wipeout occurs before finding $0^n$ with probability at least $\Omega(1)$, since a mutation generating $0^n$ has a probability of at most $1/n$, and by a union bound over $\tau$ generations and $\lambda$ offspring creations in each one, the probability of generating $0^n$ before a wipeout happens is $o(1)$. We therefore assume that a wipeout happens before $0^n$ is found, denoted as event $E_3$.

Starting from a safe population, a necessary condition for finding $0^n$ in the next $T = n/(4\lambda)$ generations without returning to a safe population is that one of the following events happens.
\begin{enumerate}
    \item During $T$ generations, the algorithm evolves a lineage from some search point $x$ containing an individual with a Hamming distance at least $n/2$ to~$x$. Since $\zeroes{x} < n/2$, this is necessary for evolving $0^n$ from a safe population. A necessary event for creating such a lineage is that all standard bit mutations executed in $T$ generations flip at least $n/2$ bits in total. Since the expected number of flipping bits in all $\lambda T$ mutations is $\lambda T = n/4$, the probability of this event is $e^{-\Omega(n)}$ by Chernoff bounds. 
    \item During $T$ generations, no wipeout happens. The probability of this event is estimated as follows. We divide those $T$ generations into $\lfloor T/\tau \rfloor$ periods of $\tau$ generations each. We then apply Lemma~\ref{lem:elimination} at the start of each period. Note that the probability of having a wipeout in any period is at least $\cwipeout$, independently of previous periods. Consequently, the probability of never experiencing a wipeout in $T$ generations is at most $(1-\cwipeout)^{\lfloor T/\tau\rfloor} = e^{-\Omega(\cwipeout \cdot T/\tau)} = e^{-\Omega(T)} = e^{-\Omega(n)}$. 
\end{enumerate}

If none of the above events occurs, the population contains only search points with fewer than $n$ zeros until the next wipeout happens.
A wipeout creates another safe population, allowing us to iterate the above argument. 
Hence, by a union bound over $e^{\varepsilon n}$ such periods, for a suitably small constant $\varepsilon > 0$, $0^n$ is not found in the next $e^{\varepsilon n}$ generations with probability $1 - e^{\varepsilon n} \cdot e^{-\Omega(n)} = 1-e^{-\Omega(n)+\varepsilon n} = 1 - e^{-\Omega(n)}$. We call this event $E_4$. By another union bound, $\Prob{E_1 \cap E_2 \cap E_3 \cap E_4} \ge 1 - \sum_{i=1}^4 (1-\Prob{E_i}) \ge 1- e^{-\Omega(n)} - o(1) - (1-\cwipeout-o(1)) - e^{-\Omega(n)} = \Omega(1)$. Thus, by the law of total probability, the expected number of generations to optimise \OTZTfull is $\Omega(1) \cdot e^{\varepsilon n} = e^{\Omega(n)}$.
\end{appendixproof}

\section{Improving SPEA2: A Refined Density Estimate}\label{sec:simplified-spea2}

We now show the opposite: \algspeaiisim with duplicate elimination can optimise
\OTZTfull efficiently with a constant population size. For this, we recall the notion of an algorithm operating on $\{0,1\}^n$ being 
$0/1$-monotone from~\cite{Dang2024Illustrating}. Informally speaking, the notion ensures that the maximum number of zeros and the maximum number of ones in the population are non-decreasing over time.

\begin{definition}[Definition~4.1 in~\cite{DangOS2025}]\label{def:01-monotone}
Consider an algorithm~$\mathcal{A}$ optimising a function $f$ on $\{0,1\}^n$ by
evolving a population $A_t$ of search points in each of iteration $t\geq 0$.
$\mathcal{A}$ is is called $0/1$-monotone if
$\forall t\geq0 \colon 
    (\exists x \in A_{t+1}\colon |x|_1 \geq \max\{|z|_1 \mid z \in A_t\}) 
    \wedge 
    (\exists x \in A_{t+1}\colon |x|_0 \geq \max\{|z|_0 \mid z \in A_t\})$. 
\end{definition}

The following lemma shows that the order induced by $\nearestdistleq{P}$ assigns
a higher rank to solutions with unique fitness vectors and to those at the 
extremes of the objective space in a population of non-dominated solutions. 

\begin{lemma}\label{lem:nearestdistlt-property}
Let $\nearestdistlt{P}$ be the order defined by Eq.~\eqref{eq:spea2-nearestdistlt} 
according to a bi-objective function $f\colon\mathcal{X}\rightarrow \R^2$ and a population $P$ of non-dominated 
search points on $\mathcal{X}$ with $|P|\geq 2$, 
and let $x,y$ be two search points in $P$. Then
\begin{enumerate}
\item $(\exists z \in P \setminus\{x\} \colon f(z)=f(x)) 
             \wedge (\forall z \in P \setminus\{y\} \colon f(z)\neq f(y)) 
             \Rightarrow x\nearestdistlt{P} y$.
\item $(|P|\geq 3) \wedge (\forall u,v \in P\colon u\neq v \Rightarrow f(u)\neq f(v)) \;\wedge\;$\\ 
            $(f_1(y)=\max_{z\in P}\{f_1(z)\}\vee f_2(y)=\max_{z\in P}\{f_2(z)\}) \;\wedge\;$\\
            $(x \text{ is a nearest neighbour of } y \text{ in } P)
             \Rightarrow x\nearestdistlt{P} y$.
\end{enumerate}
\end{lemma}
\begin{proof}
To prove (i), we argue that $(\exists z \in P \setminus\{x\} \colon f(z)=f(x))$ implies $\inearestdist{1}{P}{x}=0$, and 
    $({\forall z \in P \setminus\{y\} \colon} f(z)\neq f(y))$ implies $\inearestdist{1}{P}{y}>0$, 
thus together $x\nearestdistlt{P} y$.

For (ii), assume without loss of generality that $f_1(y)=\max_{z\in P}\{f_1(z)\}$ in the third condition 
as the case $f_2(y)=\max_{z\in P}\{f_2(z)\}$ follows symmetrically. 
Since $|P|\geq 3$, let $z$ be a nearest neighbour of $x$ in $P\setminus 
\{y\}$ ($\neq \{x\}$). Let 
    $Y:=f(y), X:=f(x)$ and $Z:=f(z)$ be the fitness vectors and 
    $a:=XY, b:=XZ$ and $c:=YZ$ be the Euclidean distances between them. 

By the second condition, all fitness vectors in $P$ are distinct, hence $X, Y, Z$ are pairwise distinct
and $a,b,c>0$. Example locations of $Y, X$, and $Z$ in the objective space are given in 
Fig.~\ref{fig:nearestdistlt-property}. 

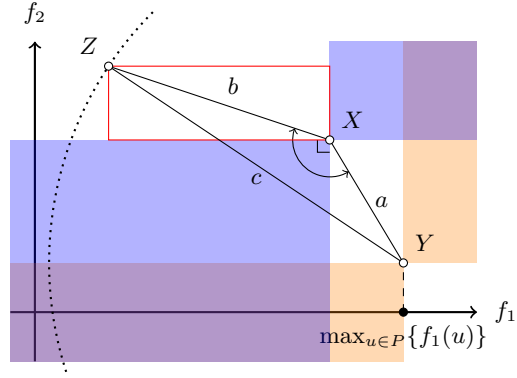
\begin{figure}[h]\centering\small
\begin{tikzpicture}[x=1em,y=1em,scale=1]
\tikzstyle{axis}=[draw,black,thick,->];
\tikzstyle{vertex}=[draw,circle,minimum size=1em,scale=0.33,fill=white,align=center];
\tikzstyle{areadomy}=[color=orange,opacity=0.3];
\tikzstyle{areadomx}=[color=blue,opacity=0.3];

\coordinate (fy)  at (15,2);
\coordinate (pfy) at (15,0);

\coordinate (fx)  at (12,7);
\coordinate (pfx) at (12,0);
\coordinate (qfx) at (0,7);

\coordinate (ra) at (12,11.5);
\coordinate (rb) at (12,-2.5);

\coordinate (fz)  at (3,10);

\draw[axis] (0,-2)  -> (0,11);
\node[label=90:{$f_2$}] at (0,11) {};
\draw[axis] (-1,0)  -> (18,0);
\node[label=0:{$f_1$}] at (18,0) {};

\fill[areadomy] (fy)  rectangle (-1,-2);
\fill[areadomy] (fy)  rectangle (18,11);
\node[vertex,fill=black,label=-90:{$\max_{u\in P}\{f_1(u)\}$}] at (pfy) {};
\draw[dashed] (pfy) -- (fy);

\fill[areadomx] (fx)  rectangle (-1,-2);
\fill[areadomx] (fx)  rectangle (18,11);

\pic[draw,angle radius=0.5em] {right angle = pfx--fx--qfx};
\pic[draw,<->,angle radius=1.5em, angle eccentricity=1.2] {angle = fz--fx--fy};
\draw[dotted,thick] 
    let \p1 = ($(fy)-(fz)$), 
        \n1 = {veclen(\x1,\y1)}
    in (fy) + (135:\n1) arc (135:198:\n1);
\draw[red] (fz) rectangle (fx);

\draw (fy) -- node[right] {$a$} (fx);
\draw (fz) -- node[above right] {$b$} (fx);
\draw (fz) -- node[below] {$c$} (fy);


\node[vertex,label=45:{$Y$}]  at (fy) {};
\node[vertex,label=45:{$X$}]  at (fx) {};
\node[vertex,label=120:{$Z$}] at (fz) {};
\end{tikzpicture}
\caption{Illustration of fitness vectors $f(y)=Y, f(x)=X$ and $f(z)=Z$.} 
\label{fig:nearestdistlt-property}
\end{figure}

Because $P$ only contains non-dominated solutions
and because of the extremal property of $y$,
$X$ cannot lie in the (orange shaded) region dominated by $Y$
nor weakly to the right of $Y$. Similarly, $Z$ 
can neither lie in any dominated (shaded) region
nor weakly to the right of $X$. 
These geometric constraints imply that the angle $\measuredangle YXZ$ 
is either flat (equal to $\pi$), or it must overly cover either the 
set of all fitness vectors dominated by $X=f(x)$ or the set of all fitness 
vectors that dominate $X$ (the two blue shaded regions on the opposite sides 
of $X$). 
Note that each of these sets is a cone with apex $X$ and opening angle 
$\pi/2$ (a right angle).
Thus, we have that either $X, Y, Z$ are on the 
same straight line,  or the triangle $\triangle XYZ$ is obtuse at $X$. 
In the former case, we get $c=a+b>\max(a,b)$, and in the latter case, 
we also have $c>\max(a,b)$ by the Triangle Larger Angle Theorem since 
edge $YZ$ of length $c$ is opposite to the largest angle 
of $\triangle XYZ$. 

We now compare nearest-neighbour distances. Since $y$'s nearest neighbour is~$x$ and $x$'s nearest neighbour is either $y$ or $z$, $\inearestdist{1}{P}{y}=a \geq \min(a,b) = \inearestdist{1}{P}{x}$. For the second nearest-neighbour distances, $\inearestdist{2}{P}{x} = \max(a, b)$ and we will show that $\inearestdist{2}{P}{y}=c$. As shown above, $c > \max(a, b)$, hence $\inearestdist{2}{P}{y} > \inearestdist{2}{P}{x}$ and  $x\nearestdistlt{P} y$.

It remains to prove $\inearestdist{2}{P}{y}=c$. Suppose, for contradiction, that there exists $z'\in P$ with $Z':=f(z')$ such that $YZ'<c$. Then $Z'$ must satisfy the same geometric constraints as $Z$, and in particular lie outside the forbidden regions,  
while also lying strictly inside the circle centred at $Y$ with radius $c$. Since neither $z\succ z'$ nor $z'\succ z$ holds, $Z'$ must lie strictly inside the axis-aligned rectangle with opposite corners $X$ and $Z$ (\eg the red rectangle in the figure). However, any such point would be closer to $X$ than $Z$, contradicting the choice of $z$ as a nearest neighbour of $x$ in $P\setminus\{y\}$. Hence, no such $z'$ exists, and $\inearestdist{2}{P}{y}=c$.
\end{proof}

The first statement of the lemma can be shown for an arbitrary number 
of objectives, and in fact is used in the literature to prove positive results 
for \algspeaii (\eg see~\cite{SPEA22024}). The second statement is novel and 
non-trivial, as its proof relies on Euclidean geometry on the plane. We now 
show that \algspeaiisim with duplicate elimination is $0/1$-monotone on \OTZT 
for a minimal archive size of~$3$.

\begin{lemma}\label{cor:spea2sim-is-01-monotone}
\algspeaiisim 
    with duplicate elimination
    and $\mu\geq 3$ 
is $0/1$-monotone on \OTZTfull.
\end{lemma}
\begin{proof}
Let $f:=\OTZT$ and consider the following two cases. 

$(A_t \cup P_t) \cap \{0^n,1^n\}=\emptyset$: No Pareto optimal solutions have
been found yet. In this case, $A'_t=A_t \cup P_t$, \ie no solution in $A_t \cup
P_t$ dominates another, thus only the truncation operator is employed. Let
$f(A'_t)$ be the set of fitness vectors of the solutions in $A'_t$, then it
follows from Lemma~\ref{lem:nearestdistlt-property}~(i) that the operator will
remove solutions that share the same fitness vector in $f(A'_t)$ one by one
until for each $v\in f(A'_t)$ there is only one $x_v\in A'_t$ where $f(x_v) =
v$. If either $\mu \geq |f(A'_t)|$ or $|f(A'_t)|\leq 2$ then it is clear that 
afterwards at least one solution with the largest fitness $f_1$, hence with 
the largest number of ones, and at least one with the largest fitness $f_2$, 
hence with the largest number of zeros, of $A_t \cup P_t$ are kept in $A_{t+1}$. 
Otherwise, if $\mu <|f(A'_t)|$ and $\mu\geq 3$, 
it follows from Lemma~\ref{lem:nearestdistlt-property}~(ii) that at the moment
that $A'_t$ is about to reduce its size from $\mu+1\geq 4>3$ to $\mu$, there 
are still sufficient slots to keep a solution with the largest $f_1$ and another
solution with the largest $f_2$ of $A_t \cup P_t$ in $A'_t$ and then to $A_{t+1}$.

$(A_t \cup P_t) \cap \{0^n,1^n\}\neq\emptyset$: Some Pareto optimal solutions have
been found. If both $0^n\in (A_t\cup P_t)$ and $1^n \in (A_t\cup P_t)$ then, owing to 
duplicate elimination, the algorithm only keeps one copy for each of these 
solution from $A_t \cup P_t$ in $A_{t+1}$, and these are the solutions with 
the largest number of ones and zeros. 
Otherwise, if only one Pareto optimal solution is found in $A_t\cup P_t$, say, $1^n$ (the argument for $0^n$ is symmetric), then again by the 
diversity mechanism, only one copy of $1^n$ is kept in $A_t \cup P_t$ and so $A'_t$ 
is the set $\{1^n\}$. Hence, a solution with the largest number of ones 
survives to $A_{t+1}$ and the (refined) completion operator of \algspeaiisim 
is employed on $P:=(A_t \cup P_t) \setminus A'_t$ to reduce its size to 
$\mu - |A'_t|=\mu -1 \geq 2$ before merging it with $A'_t$ to form $A_{t+1}$. 
Note that every solution $x$ of $P$ is dominated by~$1^n$, 
and does not dominate another solution in $P$, thus each has the identical 
raw fitness $R(x)=|P|$ by Eq.~\eqref{eq:spea2-raw-fitness}. The completion 
operator on~$P$ therefore behaves similarly to the truncation operator on $A'_t$ 
of the previous case. By almost the same arguments as in that case, we conclude 
that at least a solution with the largest number of zeros also survives to $A_{t+1}$.
A minor difference is that Lemma~\ref{lem:nearestdistlt-property}~(ii) is now 
applied at the moment of $|P|=\mu - |A'_t|+1\geq 2+1 = 3$ and this precisely 
matches the condition on $P$ from the lemma. 
\end{proof}

Once we have that \algspeaiisim is $0/1$-monotone, it is straightforward to 
show that the algorithm can fully cover the Pareto front of \OTZT in $\bigO(n\log{n})$ 
expected fitness evaluations using constant population sizes. The theorem 
adapts the results of~\cite{Dang2024Illustrating,DangOS2025} but for the 
($\mu$+$\lambda$) scheme of \algspeaiisim.

\begin{theoremrep}\label{thm:spea2sim}
For every $\mu\geq 3$ and $\lambda\geq 1$, \algspeaiisim 
    with duplicate elimination and standard bit mutation
optimises $\OTZTfull$ in $\bigO(\lambda n + \mu n\log{n})$ fitness evaluations 
in expectation.
\end{theoremrep}
\begin{appendixproof}
By Lemma~\ref{cor:spea2sim-is-01-monotone}, solutions with the largest
number of ones and zeros seen so far are always kept in the archive $A_t$, owing to 
the diversity mechanism and the sufficiently large population size. 
Hence, it suffices to bound the expected time to reach either $1^n$ or $0^n$. By symmetry, the same bound applies to both targets, and the expected time to reach both strings is at most the sum of the two, which does not affect the asymptotic bound.

We choose the target $1^n$, and let $i:=\max\{|x|_1\mid x\in A_t\}$ be
the largest number ones in the current archive at time $t$. In one offspring 
production, to create an individual with more $1$s, it suffices to pick one 
individual with $i$ ones from $A_t$, then flip one of its zeros while keeping 
the rest of its bits unchanged, and this event occurs with probability 
$\frac{1}{\mu}\cdot {n-1 \choose 1} \frac{1}{n}\left(1-\frac{1}{n}\right)^{n-1}\geq\frac{n-i}{e\mu n}=:s_i$ 
for the standard bit mutation. This is repeated $\lambda$ times 
independently to create $P_t$, thus the probability of creating at least an 
individual with more than $i$ ones in $P_t$ is at least 
$1-(1-s_i)^{\lambda}\geq \frac{\lambda s_i}{\lambda s_i + 1}=:q_i$ by Lemma~10
of~\cite{Badkobeh2015}. The initial population has $\mu$ individuals evaluated
and in each generation there are $\lambda$ evaluations, so the expected number 
of fitness evaluations to create $1^n$ is at most
\begin{align*}
\mu + \lambda \sum_{i=0}^{n-1}\frac{1}{q_i}
= \mu + \lambda \sum_{i=0}^{n-1}\left(1 + \frac{1}{\lambda} \cdot \frac{e\mu n}{n-i}\right)
\leq \mu + \lambda n + \mu en H_n, 
\end{align*}
which is $\bigO(\lambda n + \mu n \log{n})$. 
This is also an asymptotic bound on the expected running time of \algspeaiisim 
to fully cover the Pareto front of the function.  
\end{appendixproof}

In \algspeaiisim, we only refine the part of \algspeaii that was largely 
ignored in the previous rigorous studies, therefore, the following already 
proven results for the original algorithm automatically hold for the variant. 
The definitions of the related functions can be found in the referenced papers~\cite{DoerrKS2026,SPEA22024}. 

\begin{theorem}[Theorems of~\cite{DoerrKS2026,SPEA22024}]\label{thm:spea2sim-from-literature}
The following expected runtime bounds hold for both (vanilla) \algspeaii and \algspeaiisim. 
The last three results assume $\lambda=\bigO(\mu)$ while the first results consider the 
general choice $\lambda\geq 1$.
\begin{enumerate}
\item $\bigO((\mu+\lambda)n + n^2\log{n})$ on \OMMfull with $\mu\geq n+1$.
\item $\bigO((\mu+\lambda)n\log(\frac{\mu}{n+1}) + n^3 + \lambda n)$ on \LOTZfull with $\mu\geq n+1$.
\item $\bigO((\mu+\lambda)n + n^{k+1})$ on \OJZJfull with $\mu \geq n-2k+3$.
\item $\bigO(\mu n \min(m\log{n}, n))$ on \mOMMfull with $\mu\geq \left(\frac{2n}{m}+1\right)^{\frac{m}{2}}$.
\item $\bigO(\mu n^2)$ on \mLOTZfull with $\mu\geq \left(\frac{2n}{m}+1\right)^{m-1}$.
\item $\bigO(\mu n^k \min(mn, 3^{\frac{m}{2}}))$ on \mOJZJfull with $\mu\geq \left(\frac{2n}{m}-2k+3\right)^{\frac{m}{2}}$.
\end{enumerate}
\end{theorem}

\section{Experiments}\label{sec:experiments}


We conducted experiments on the studied algorithms \algnsga, \algnsgaiii, \algsms, \algspeaii 
and \algspeaiisim on \OTZTfull to complement the theoretical findings.
\ifreview
\else 
We use the code base\footnote{At \url{https://gitlab.com/d2cmath/emoas-pareto-sparse}.} of the previous work~\cite{DangOS2025}, which already implemented most of the mentioned algorithms and the function in Python using PyMOO library~\cite{Deb2020}. We only added our implementation\footnote{PyMOO also provides \algspeaii, however, at the time of publishing this paper, the library implementation of the truncation operator is incorrect, compared to the original paper~\cite{SPEA2original}. We reported this issue (\ie \href{https://github.com/anyoptimization/pymoo/issues/773}{issue~773}) and suggested a correction.} 
of \algspeaii and \algspeaiisim for the experiments. We have also made our code publicly available\footnote{At \url{https://gitlab.com/d2cmath/spea2-plus}. 
We also experimented with another function of~\cite{DangOS2025}, in which both \algspeaii and \algspeaiisim outperform the other algorithms. 
\ifarxiv
Those results are omitted here but available on the link.
\else
Due to space restriction, those results are omitted here but available on the link.
\fi}. 
\fi

Tab.~\ref{tab:config-experiments} summarises the configuration of the
experiments. We considered \OTZT on bit strings of length
$n=80$ with varying parent population size $\mu$, and recorded the success
rate of covering the Pareto front within $n^3=5.12\times 10^5$ fitness
evaluations over $100$ replicated runs. Duplication elimination is enabled in 
all runs. 
The offspring population size of \algspeaii and \algspeaiisim is set to
$\lambda=\mu$, and this makes them comparable to \nsga and \nsgaiii. 
With this setting, it turns out that 
\algspeaii requires a quite large $\mu$ to achieve a success rate of $100\%$,
thus we tested $\mu\in\{5,10,\ldots,180\}$. 
We considered uniform crossover with parameter $p_c\in\{0.0,0.5\}$, 
\ie to turn off and on the use of this operator. 

\begin{table}[ht]\centering\footnotesize
\begin{tabular}{ll}
\toprule
\textbf{Configuration} & \textbf{Setting}\\
\midrule
Algorithms & \algnsga, \algnsgaiii, \algsms, \algspeaii, \algspeaiisim\\
Problems & \OTZT with $n=80$ \\
Parent pop. size & $\mu\in\{5,10,\ldots,180\}$ \\
Offspring pop. size & $\lambda=\mu$ for all algorithms except $\lambda=1$ for \algsms\\
Parent selection & Uniform with replacement\\
Mutation & Standard bit mutation\\
Crossover & Uniform crossover with $p_c\in\{0.0, 0.5\}$\\
Diversity mechanism & Duplicate elimination (removal of genotype copies)\\
Max. fitness eval. & $5.12\times 10^{5}$ ($=n^3$)\\
No. of replications & $100$\\
\bottomrule
\end{tabular}
\caption{Configuration of the experiments.}
\label{tab:config-experiments}
\end{table}

Fig.~\ref{fig:experiments-otzt} summarises our results for \OTZTfull. As our
theory predicts, \algnsga, \algnsgaiii, \nsgaiii and \algspeaiisim have $100\%$
success rate to fully cover the Pareto front for all population sizes tested
in both settings of enabling and disabling crossover  
(except a tiny dip for \algsms when crossover is enabled for $\mu=170$). 
In contrast, the original \algspeaii struggles to optimise \OTZT when 
the population size is too small, \eg the success rate for the budget of 
$5.12\times 10^{5}$ fitness evaluations is below $10\%$ for $\mu\leq 50$. 
In order to reach the success rate of $100\%$ for $n=80$, \algspeaii requires 
a population size of $\mu\geq 125$.

\begin{figure}[ht]
\centering\hspace*{-10pt}
\begin{subfigure}[b]{0.495\textwidth}\small
    \centering
    \includegraphics[width=1.02\textwidth]{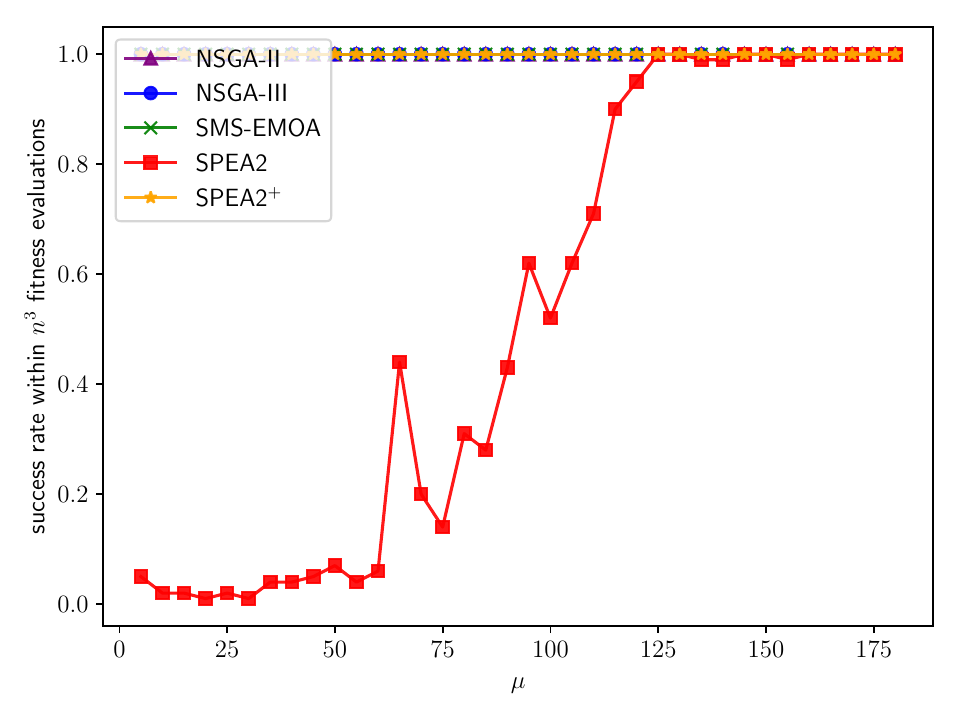}
    \caption{$p_c=0$}
    \label{fig:otzt-pc00}
\end{subfigure}
\hfill\hspace*{-5pt}
\begin{subfigure}[b]{0.495\textwidth}\small
    \centering
    \includegraphics[width=1.02\textwidth]{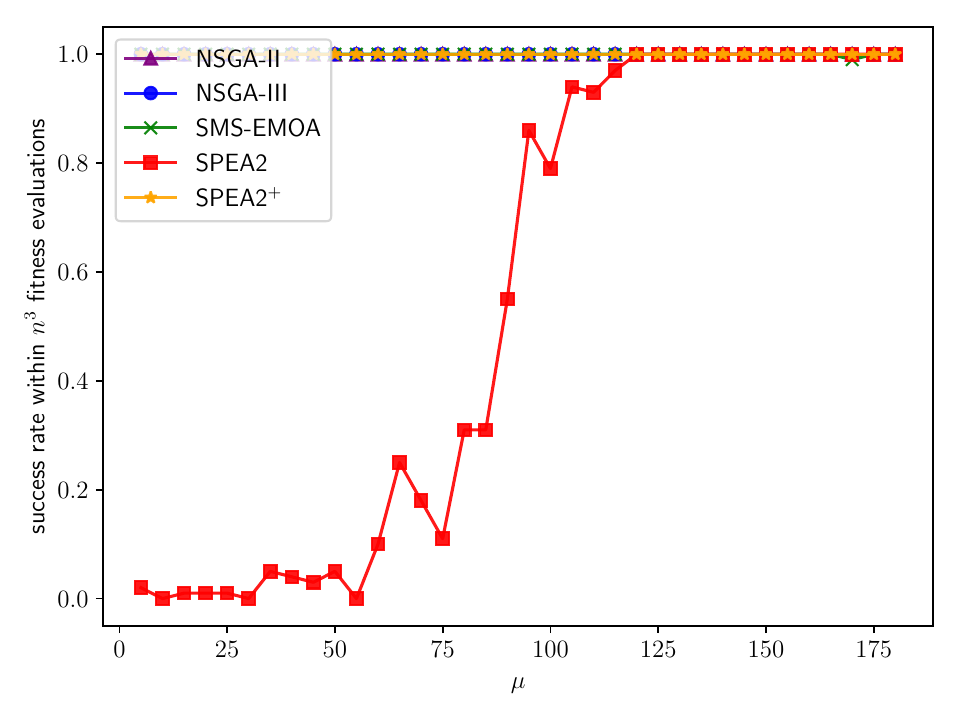}
    \caption{$p_c=0.5$}
    \label{fig:otzt-pc05}
\end{subfigure}
\caption{\small Success rate over $100$ runs to cover the Pareto front on \OTZTfull with $n=80$ with a computational budget of $n^3=5.12\times 10^{5}$ fitness evaluations and varying $\mu\in\{5,10,\ldots,180\}$. The results with uniform crossover disabled are on the left, and those with crossover enabled at rate $p_c=0.5$ are on the right.}
\label{fig:experiments-otzt}
\end{figure}

\section{Conclusions}\label{sec:concl}

We discovered a weakness of \algspeaii arising from its fitness assignment based on the $k$-nearest neighbour density estimate. 
It leads to inefficiency in covering the Pareto front of \OTZTfull
when using a constant population size, even with duplicate elimination.
This stands in contrast to the efficiency of other prominent algorithms, such as
like \algnsga, \algnsgaiii, and \algsms~\cite{Dang2024Illustrating,DangOS2025}, 
under the same setting. 
\algspeaii is inefficient on \OTZTfull as the $k$\nobreakdash-th nearest-neighbour distance fails to provide a sufficient signal for maintaining diversity within the dominated regions of the objective space.
To address this issue, we proposed an improved variant 
called \algspeaiisim that considers all pairwise distances, along with the Pareto strength. 
\algspeaiisim is efficient on \OTZTfull while preserving the established performance guarantees of the original \algspeaii on simpler problems.

Our experiments confirm a large performance gap between 
the original \algspeaii and the other algorithms, including \algspeaiisim, 
for constant parent population sizes.  
%
Future work should investigate similar points further, in particular to deepen
our understanding of the role of the Pareto-strength component in~\algspeaii.

\bibliographystyle{abbrvnat}
\bibliography{references} 


\end{document}